 \documentclass[american,letterpaper]{article} 
\usepackage{fullpage}
\usepackage[hyphens]{url}  
\usepackage{graphicx} 
\usepackage[square,sort,numbers]{natbib}  
\usepackage{caption} 
\usepackage{algorithm}
\usepackage{algorithmic}
\usepackage{xcolor}

\usepackage[affil-it]{authblk}

\usepackage{subfigure}
\usepackage{booktabs} 

\usepackage{multirow}

\usepackage{graphics}

\newcommand{\R}{\mathbb{R}}

\newcommand{\p}{{\rm I}\kern-0.18em{\rm P}}
\newcommand{\E}{{\rm I}\kern-0.18em{\rm E}}

\usepackage{amsmath}
\usepackage{amssymb}
\usepackage{mathtools}
\usepackage{amsthm}

\usepackage{newfloat}
\usepackage{listings}
\DeclareCaptionStyle{ruled}{labelfont=normalfont,labelsep=colon,strut=off} 
\floatstyle{ruled}
\newfloat{listing}{tb}{lst}{}
\floatname{listing}{Listing}
\newenvironment{noindlist}
 {\begin{list}{\labelitemi}{\leftmargin=0em \itemindent=1em}}
 {\end{list}}

\title{mSAM: Micro-Batch-Averaged Sharpness-Aware Minimization}

\author[1,2]{Kayhan Behdin\thanks{behdin1675@gmail.com}}
\author[1]{Qingquan Song}
\author[1]{Aman Gupta}
\author[1]{Sathiya Keerthi}
\author[1]{Ayan Acharya}
\author[1]{Borja Ocejo}
\author[1,3]{Gregory Dexter}
\author[3]{Rajiv Khanna}
\author[1]{David Durfee}
\author[1,2]{Rahul Mazumder}
\affil[1]{LinkedIn Corporation, USA}
\affil[2]{Massachusetts Institute of Technology, USA}
\affil[3]{Purdue University, USA}

\date{}

\usepackage{bibentry}

\begin{document}

\maketitle

\begin{abstract}

Modern deep learning models are over-parameterized, where different optima can result in widely varying generalization performance. The Sharpness-Aware Minimization (SAM) technique modifies the fundamental loss function that steers gradient descent methods toward flatter minima, which are believed to exhibit enhanced generalization prowess. Our study delves into a specific variant of SAM known as micro-batch SAM (mSAM). This variation involves aggregating updates derived from adversarial perturbations across multiple shards (micro-batches) of a mini-batch during training. 
We extend a recently developed and well-studied general framework for flatness analysis to theoretically show that  SAM achieves flatter minima than SGD, and mSAM achieves even flatter minima than SAM. We provide a thorough empirical evaluation of various image classification and natural language processing tasks to substantiate this theoretical advancement. We also show that contrary to previous work, mSAM can be implemented in a flexible and parallelizable manner without significantly increasing computational costs. Our implementation of mSAM yields superior generalization performance across a wide range of tasks compared to SAM, further supporting our theoretical framework.

\end{abstract}

\newcommand{\Scal}{\mathcal{S}}
\newcommand{\Abhat}{\hat{\matr{A}}}
\newcommand{\EE}{\mathbb{E}}

\section{Introduction}

Overparameterized deep neural networks (DNNs) have established themselves as a cornerstone of modern advancements in machine learning, consistently delivering state-of-the-art results across diverse domains such as image comprehension~\cite{alexnet, resnets, tan2019efficientnet}, natural language processing (NLP)~\cite{vaswani2017attention, devlin2018bert, liu2019roberta}, and recommender systems~\cite{guo2017deepfm, naumov2019deep}. The training of DNNs necessitates the minimization of complex and non-convex loss functions, entailing a multitude of minima. Intriguingly, these distinct minima can exhibit varying degrees of generalizability when faced with previously unseen data~\cite{flatness2, liu2020bad}. Consequently, the selection of an optimization framework capable of identifying minima that contribute to robust generalization performance assumes paramount significance. A broad spectrum of optimization algorithms has been developed to cater to diverse domains, including methodologies like stochastic gradient descent (SGD), heavy-ball momentum~\cite{sutskever2013importance}, Adam~\cite{kingma2014adam}, and LAMB~\cite{you2019large}, among others. When complemented with appropriate regularization techniques, these approaches play a pivotal role in yielding robust generalization capabilities. 
The ability to perform implicit regularization of SGD-like methods has also garnered considerable attention in recent years ~\cite{jasf20, cokl21}.

In recent times, a substantial body of work has been dedicated to exploring the relationship between the geometry of the loss landscape and its impact on generalization~\cite{flatness2, flatness1, xiss21, wuws22, hawl21, saqu18}. The novel Sharpness-Aware Minimization (SAM) algorithm~\cite{sampaper} capitalizes on this interplay by introducing adjustments to the loss function that enable the optimization process to gravitate towards solutions characterized by increased flatness during training, resulting in enhanced generalization across a broad spectrum of tasks and domains. Specifically, this approach transforms the structure of the loss function to account for the maximal value within a localized vicinity surrounding the current parameters within the loss landscape. The gradient descent step within this framework involves a composite of two distinct phases, the first applying an adversarial perturbation. Conceptually, SAM gauges flatness by analyzing gradients in an adversarial direction, guiding a descent strategy based on the sharpness at its worst case.

Unlike standard gradient descent (GD), the compositional nature of the SAM gradient computation implies that further splitting the mini-batch into disjoint shards (micro-batches) and averaging the updates will not lead to the same gradient. This was noted in \citet{sampaper}, where they introduced this notion as mSAM, with $m$ being the number of micro-batches of a single mini-batch. By splitting into disjoint shards, mSAM leverages several adversarial directions, which may more robustly represent flatness.
In this work, we build the corresponding theoretical framework for improved flatness of mSAM that utilizes and extends previous techniques. Our extensive experimental results confirm these insights on a wide range of tasks. While the focus will particularly be on mSAM, our techniques apply more generally. They further imply that splitting the mini-batch for compositional gradient computations, utilized in other variants of SAM, may also lead to an improved flatness of minima. Such observations invariably encourage further study of splitting the mini-batch into shards whenever the gradient update follows a non-linear aggregation.

\paragraph{Related Work:}

Although the sharpness of the loss landscape can be calculated using several different measures, such as the largest eigenvalue~\cite{wucw18}, trace~\cite{minimumsharpness}, or Frobenius norm~\cite{wuws22} of the loss Hessian, many of these metrics prove to be computationally demanding for practical purposes. Given the established correlation between the sharpness of the loss landscape and generalization performance~\cite{flatness2}, the central concept underlying the SAM algorithm~\cite{sampaper} revolves around guiding the network to explore regions where the worst-case loss value within a local neighbourhood remains reasonably moderate. This pragmatic approximation of sharpness offers a computationally manageable alternative, distinct from the sharpness metrics mentioned before. The emergence of SAM has sparked a wave of interest in sharpness-aware training, culminating in the development of several variants~\cite{gsampaper, looksam, samforfree, kwkp21, kilh22}.

In this paper, we focus on the effect of further splitting the mini-batch into shards for gradient computation of sharpness-aware minimization methods. While our theoretical results on flatness will generalize, we primarily consider mSAM, which is this technique applied to the original SAM algorithm. In~\citet{sampaper}, mSAM is used \textit{implicitly} to reduce the computational cost of SAM by avoiding synchronization across multiple GPUs (referred to as ``accelerators'' hereinafter). Recently, it has been observed via limited experimentation that mSAM results in better generalization performance~\cite{sampaper,samvit,nonsensemsam}. \citet{icmlPaper} present mathematical expressions for mSAM, but their analysis is primarily focused on a particular version of mSAM (see Section~\ref{sec:algs} for more details). The experiments are also limited to image classification tasks on small architectures. This paper provides a more general theoretical framework that can also be extended to other sharpness-aware minimization variants.

A cluster of contemporary research papers centered on the stability analysis of GD and SGD-like algorithms unveil that they function in a regime teetering on the edge of stability~\cite{cokl21, cogk22, arlp22, jasf20}. This precarious balance is characterized by the maximum eigenvalue of the Hessian matrix associated with the training loss converging near the threshold of $2/\eta$, where $\eta$ denotes the learning rate. These findings subsequently offer a foundation upon which upper bounds on the maximum eigenvalue of the Hessian can be formulated (further elaboration is provided in Section~\ref{sec:theory}).

\paragraph{Our Contributions:}

Our contributions in this paper can be summarized as follows:
\begin{noindlist}
\item We demonstrate how mSAM improves flatness over SAM, which in turn guarantees better flatness than SGD. To that end, we leverage theoretical ideas about the implicit generalization ability of SGD-like methods and recent work related to the stability analysis of full-batch GD and SGD.
\item Starting from the mathematical description of mSAM, we present an \textit{explicit} and flexible implementation of mSAM that does not rely on accelerator synchronization and is compatible with any single/multi-accelerator setup.
\item We conduct extensive experiments on a wide variety of computer vision and NLP tasks, leveraging architectures like Convolutional Neural Networks (CNNs)~\cite{alexnet, resnets} and Transformers~\cite{vaswani2017attention}, where mSAM consistently outperforms SAM and vanilla training strategy.
\end{noindlist}

\section{Algorithm}\label{sec:algs}

In this section, we rigorously introduce mSAM, based on the SAM algorithm that aims to obtain flat solutions to the empirical loss function. In particular, SAM tries to find a solution that minimizes the worst-case loss in a ball around the solution. Mathematically, let $\mathcal{T}=\{(\mathbf{x}_i,y_i), i\in[n]: \mathbf{x}_i\in\mathcal{X},y_i\in\mathcal{Y}\}$ be a training dataset of $n$ samples, where $\mathcal{X}$ is the set of features, $\mathcal{Y}$ is the set of outcomes and $[n]=\{1,\ldots,n\}$. Moreover, let $\ell:\R^d\times \mathcal{X}\times \mathcal{Y}\mapsto \R$ be a differentiable loss function, where $d$ is the number of model parameters. Let $\mathcal{S}\subset [n]$ be a randomly chosen mini-batch of size $B$.
The empirical loss over the mini-batch $\mathcal{S}$ is defined as $ \mathcal{L}_{\mathcal{S}}(\mathbf{w})=\sum_{i\in \mathcal{S}} \ell(\mathbf{w};\mathbf{x}_i,y_i)/B$, where $w$ parameterizes the neural network. With this notation in place, the SAM loss function is defined as~\cite{sampaper}:
\begin{equation}\label{sam-loss}
    \mathcal{L}^{SAM}_{\mathcal{S}}(\mathbf{w})=\max_{\|\mathbf{\epsilon}\|_p\leq \rho}\mathcal{L}_{\mathcal{S}}(\mathbf{w}+\mathbf{\epsilon})
\end{equation}
for some $p\geq 1$. In this work, we use $p=2$. In practice, however, the maximization step in~\eqref{sam-loss} cannot be done in closed form. Hence, authors in \cite{sampaper} use a first-order approximation to $\mathcal{L}_S$ to simplify~\eqref{sam-loss} as
\begin{equation}\label{sam-loss-linear}
  \mathcal{L}^{SAM}_{\mathcal{S}}(\mathbf{w})\approx\max_{\|\mathbf{\epsilon}\|_2\leq \rho}\mathcal{L}_{\mathcal{S}}(\mathbf{w})+\mathbf{\epsilon}^{T} \nabla\mathcal{L}_S(\mathbf{w}).  
\end{equation}
It is easy to see that the maximum in Problem~\eqref{sam-loss-linear} is achieved for
\begin{equation}
    \hat{\mathbf{\epsilon}}=\rho \nabla\mathcal{L}_{\mathcal{S}}(\mathbf{w})/\|\nabla\mathcal{L}_{\mathcal{S}}(\mathbf{w})\|_2.
\end{equation}
As a result, $\mathcal{L}_{\mathcal{S}}^{SAM}\approx \mathcal{L}_{\mathcal{S}}(w+\hat{\epsilon})$. This leads to the gradient  
$$\nabla\mathcal{L}_{\mathcal{S}}^{SAM}(\mathbf{w})\approx\nabla_{\mathbf{w}}\left[\mathcal{L}_{\mathcal{S}}(\mathbf{w}+\hat{\mathbf{\epsilon}})\right]=\frac{\partial(\mathbf{w}+\hat{\mathbf{\epsilon}})}{\partial \mathbf{w}}\nabla\mathcal{L}_{\mathcal{S}}(\mathbf{w}+\hat{\mathbf{\epsilon}}).
$$
However, calculating $\partial(w+\hat{\epsilon})/\partial w$ involves second-order terms that require access to Hessian, which can be computationally inefficient in practice. Thus, by ignoring the second-order terms in the above approximation, the gradient of the SAM loss can be approximated as~\cite{sampaper}:
\begin{equation}\label{sam-gradient}
    \nabla \mathcal{L}^{SAM}_{\mathcal{S}}(\mathbf{w}) \approx \nabla \mathcal{L}_{\mathcal{S}}(\mathbf{w}+\rho \nabla\mathcal{L}_{\mathcal{S}}(\mathbf{w})/\|\nabla\mathcal{L}_{\mathcal{S}}(\mathbf{w})\|_2)
\end{equation}
which is used in the SAM algorithm (for example, in conjunction with SGD). We refer to~\citet{sampaper} for more details and intuitions about SAM. We call the inner gradient calculations on the right-hand side of~\eqref{sam-gradient} as the SAM ascent step and the outer gradient calculations as the gradient step. mSAM~\cite{sampaper} is a variation of the SAM algorithm. In general, for mSAM, a mini-batch of data  $\mathcal{S}$ is further divided into $m$ smaller disjoint shards (aka ``micro-batches''), such as $\mathcal{S}_1,\cdots,\mathcal{S}_m$ where $\cup_{j=1}^m \mathcal{S}_j=\mathcal{S}$. For simplicity, we assume $|\mathcal{S}_1|=\cdots=|\mathcal{S}_m|=|\mathcal{S}|/m$ although such an assumption is not necessary in general. The mSAM loss is a variation of the SAM loss, defined as:
\begin{equation}\label{msam-loss}
     \mathcal{L}^{mSAM}_{\mathcal{S}}(\mathbf{w}) = \frac{1}{m}\sum_{j=1}^m \max_{\|\mathbf{\epsilon}^{(j)}\|_2\leq \rho} \mathcal{L}_{\mathcal{S}_j}(\mathbf{w}+\mathbf{\epsilon}^{(j)}).
\end{equation}
Intuitively, mSAM is a version of SAM where the ascent step (or the weight perturbation) of SAM is done independently on each micro-batch using different $\epsilon^{(j)}$, instead of using an average perturbation such as $\mathbf{\epsilon}$ for all micro-batches. The mSAM gradient can thereby be derived as:
\begin{equation}\label{msam-gradient}
        \nabla \mathcal{L}^{mSAM}_{\mathcal{S}}(\mathbf{w}) = \frac{1}{m}\sum_{j=1}^m \nabla \mathcal{L}_{\mathcal{S}_j}(\mathbf{w}+\rho \frac{\nabla\mathcal{L}_{\mathcal{S}_j}(\mathbf{w})}{\|\nabla\mathcal{L}_{\mathcal{S}_j}(\mathbf{w})\|_2}),
\end{equation}
where~\eqref{msam-gradient} is a first-order approximation to the gradient of~\eqref{msam-loss}. We also note that the loss~\eqref{msam-loss} is  related to the mSAM definition of~\citet{icmlPaper}. See Table~\ref{sam-table} for a side-by-side comparison of SAM and mSAM and their different implementations.

\begin{table*}[t]
\renewcommand{\arraystretch}{1.2}
\begin{center}
\begin{sc}
\begin{tabular}{ccccc}
\toprule
 & SAM &  \multicolumn{3}{c}{mSAM} \\
 \midrule
 Loss function & $\max_{\|\mathbf{\epsilon}\|_2\leq \rho} \sum_{i=1}^m  \mathcal{L}_{\mathcal{S}_i}(\mathbf{w}+\mathbf{\epsilon})/m$ & \multicolumn{3}{c}{$\sum_{i=1}^m \max_{\|\mathbf{\epsilon}^{(i)}\|_2\leq \rho} \mathcal{L}_{\mathcal{S}_i}(\mathbf{w}+\mathbf{\epsilon}^{(i)})/m$}\\
 Ascent step & $\hat{\mathbf{\epsilon}}\propto \rho {\sum_{i=1}^m \nabla\mathcal{L}_{\mathcal{S}_i}(\mathbf{w})}/m$ &  \multicolumn{3}{c}{$\hat{\mathbf{\epsilon}}^{(i)}\propto\rho{ \nabla\mathcal{L}_{\mathcal{S}_i}(\mathbf{w})},~i\in[m]$} \\
  Gradient & $g= \sum_{i=1}^m\nabla \mathcal{L}_{\mathcal{S}_i}(\mathbf{w}+\hat{\mathbf{\epsilon}})/m$ &  \multicolumn{3}{c}{$g= \sum_{i=1}^m \nabla \mathcal{L}_{\mathcal{S}_i}(\mathbf{w}+\hat{\mathbf{\epsilon}}^{(i)})/m$} \\
 Implementations & [FKMN] & [FKMN] & [AF] & Ours \\
Possible $m$ values & - & $\#$ of accelerators & flexible & flexible \\
Processor support & Multiple & Multiple & Single & Multiple\\
\bottomrule
\end{tabular}
\end{sc}
\end{center}
\caption{Comparison of SAM with Different mSAM Implementations. [FKMN] refers to~\citet{sampaper} and [AF] refers to~\citet{icmlPaper}.}
\label{sam-table}
\end{table*}

An important distinction between our work and prior work is that we treat $m$ as a model hyper-parameter to improve generalization. In particular, in mSAM implementation of~\citet{sampaper}, the value of $m$ is fixed to the number of hardware accelerators, micro-batch $i$ is the part of the data that is loaded onto accelerator $i$, and each accelerator uses a separate perturbation, simulating the effect of mSAM. With this implementation, $m$ is an artefact of the hardware setup. On the other hand, the analysis of~\citet{icmlPaper} mostly concerns the value $m = |\mathcal{S}| := B$ where $\mathcal{S}$ is the mini-batch under consideration, and we denote its size as $B$ for ease of use in latter sections. In contrast, we consider a wide range of values for $m$ in our experiments. This offers the flexibility to choose an appropriate value of $m$ that leads to a better generalization performance. Moreover, our implementation supports any single/multi-accelerator setup and allows the user to set an appropriate value of $m$.

\section{Justification of mSAM}
\label{sec:theory}

\def\M{{\cal{M}}}
\def\lmax{\lambda_{\max}}
\newcommand{\defeq}{\overset{\mathrm{def}}{=\joinrel=}}
\def\etahat{\hat{\eta}}
\def\Ahat{{\mathbf{X}}}

\newtheorem{theorem}{Theorem}[section]
\newtheorem{assumption}{Assumption}[section]
\newtheorem{lemma}{Lemma}[section]
\newtheorem{remark}{Remark}[section]

Over-parameterized DNNs have a continuum (a manifold of large size) of minima, $\M$ due to the number of parameters being much larger than the number of training examples. An intriguing property of DNNs is that different minima in $\M$ have different sharpness values. When trained using SGD with a large learning rate and small batch size, these DNNs have an implicit ability to move towards minima which are flat - or equivalently, \textit{less sharp}, with sharpness expressed as the spectral norm of the Hessian\footnote{Sharpness can also be quantified in other ways. See \cite{wucw18, wuws22, jasf20, cogk22} for some details.}~\cite{wucw18, wuws22}. \citet{jasf20, cokl21} establish a theoretical framework to explain this phenomenon with strong empirical backing. Accordingly, this approach is followed by most recent theoretical papers on analyzing the properties of GD, SGD, and SAM \cite{cokl21, cogk22, balb22, weml22, ujtk22}. In this section, we review a generalized version of such analyses in SGD and extend it to SAM and mSAM. This extension of the well-studied framework implies that SAM improves flatness over SAM, and mSAM improves flatness even further than SAM. The proofs of all results in this section are given in appendix~\ref{sec:proofs3}.

\subsection{Analysis of Linear Stability}
\label{subsec:linstab}

Our analysis will focus on the general minibatch stochastic dynamic update of the following form:
\begin{equation} \label{eq:wup}
\mathbf{w}_{(t+1)} = \mathbf{w}_t - \eta \; \mathbf{d}_{\mathcal{S}}(\mathbf{w}_t) 
\end{equation}
The trio of methodologies under scrutiny — SGD, SAM and mSAM — all derive from the common foundation of $\nabla\mathcal{L}_{\mathcal{S}}$, yet they adopt distinct formulations for $\mathbf{d}_{\mathcal{S}}$. Given their shared objective to minimize the training loss, each minimum within the continuum $\mathcal{M}$ can be posited as an equilibrium point of the dynamic system (\ref{eq:wup}). Let $\mathbf{w}^*$ represent a minimum within this continuum. It is anticipated that the stability attributes of trajectories surrounding $\mathbf{w}^*$ will depend on hyperparameters, specifically the learning rate $\eta$ and the mini-batch size $B$. Furthermore, due to inherent differences in their dynamics, the three methods mentioned above will manifest varying stability characteristics even with identical values of $\eta$ and $B$. Establishing these properties hinges upon an extension of an approach to stability introduced by \citet{wucw18}.

In particular, \citet{wucw18} use a linear approximation of the gradient (equivalently, a quadratic approximation of the loss) around $\mathbf{w}^*$:
\begin{equation} \label{lingapprox}
\nabla\mathcal{L}_i(\mathbf{w}) \approx \mathbf{H}_i(\mathbf{w}^*) (\mathbf{w} - \mathbf{w}^*),
\end{equation}
where the subscript $i$ corresponds to the  $i^{\text{th}}$ training example.
For the sake of simplicity in notation, and without sacrificing generality, we assume that $\mathbf{w}^* = \mathbf{0}$. Incorporating the approximation (\ref{lingapprox}) into the construction of $\mathbf{d}$ as outlined in (\ref{eq:wup}) yields a reformulation of the following manner:
\begin{equation}\label{linup}
\mathbf{w}_{(t+1)} = (\mathbf{I} - \eta \mathbf{J}_{\mathcal{S}}(\mathbf{w}^*) ) \; \mathbf{w}_t
\end{equation}
The comprehensive derivation of $\mathbf{J}_{\mathcal{S}}(\mathbf{w}^*)$ for all three methods are presented in section~\ref{subsec:stmethods}, with further elaboration available in appendix~\ref{sec:proofs3}. 
We assume that $\mathbf{J}_{\mathcal{S}}(\mathbf{w}^*)$ is symmetric and positive semi-definite (PSD). The corresponding stability concept as outlined by \citet{wucw18} is articulated by examining the following expression:
\begin{equation}
    \E\|\mathbf{w}_{(t+1)}\|^2 = \mathbf{w}_t^T \E [(\mathbf{I} - \eta \mathbf{J}_{\mathcal{S}}(\mathbf{w}^*) )^2] \; \mathbf{w}_t
\end{equation}
To ensure that $\E\|\mathbf{w}_{(t+1)}\|^2 \le \|\mathbf{w}_t\|^2 \; \forall \mathbf{w}_t$, we must have $\lambda_1 (\E [(\mathbf{I} - \eta \mathbf{J}_{\mathcal{S}}(\mathbf{w}^*) )^2] ) \le 1$, where $\lambda_1(\mathbf{A})$ is the spectral norm of matrix $\mathbf{A}$.

\begin{lemma}
$\E [(\mathbf{I} - \eta \mathbf{J}_{\mathcal{S}}(\mathbf{w}^*) )^2] = (\mathbf{I} - \eta \mathbf{J}^* )^2 + \eta^2 \mathbf{\Sigma}$, where $\mathbf{\Sigma}= \E [(\mathbf{J}_{\mathcal{S}}(\mathbf{w}^*) - \mathbf{J}^*)^2]$ and $\mathbf{J}^* = \E [\mathbf{J}_{\mathcal{S}}(\mathbf{w}^*)]$.
\end{lemma}

Therefore, stability holds iff
\begin{equation} \label{eq:stcheck1}
    \lambda_1((\mathbf{I} - \eta \mathbf{J}^* )^2 + \eta^2 \mathbf{\Sigma}) \le 1
\end{equation}
For the full batch case, $\mathbf{\Sigma}=\mathbf{0}$, rendering the verification of 
(\ref{eq:stcheck1}) straightforward (see Lemma 3.2 below). However, for the mini-batch case, due to the complexity of analyzing (\ref{eq:stcheck1}), \citet{wucw18} opt to substitute the stability evaluation, (\ref{eq:stcheck1}), with the following necessary condition:
\begin{equation} \label{eq:neccond}
    S_1: \lambda_1((\mathbf{I} - \eta \mathbf{J}^* )^2) \le 1 \mbox{ and } S_2: \eta^2\lambda_1(\mathbf{\Sigma}) \le 1,
\end{equation}
which we consider hereafter to define stability. Since $\lambda_1(\mathbf{A} + \mathbf{B}) \geq \max\{\lambda_1(\mathbf{A}), \lambda_1(\mathbf{B})\}$ holds true for any pair of symmetric PSD matrices $\mathbf{A}$ and $\mathbf{B}$, the condition stated in (\ref{eq:neccond}) can be considered as a relaxation of the condition in (\ref{eq:stcheck1}).

\begin{lemma}
\label{eos-sgd}
$S_1$ is equivalent to checking if $\eta \lambda_1(\mathbf{J}^*) \le 2$.
\end{lemma}

The following result is a corollary of (\ref{eq:neccond}) and Lemma 3.2, and it plays a pivotal role in substantiating the principal conclusions of subsection~\ref{subsec:stmethods}.

\begin{lemma}
The stability condition of (\ref{eq:neccond}) will be violated (i.e., the equilibrium point, $\mathbf{w}^*$ is considered to be unstable for (\ref{linup})) iff
\begin{equation} \label{eq:instab}
    \eta \lambda_1(\mathbf{J}^*) > 2 \mbox{ or } \eta^2\lambda_1(\mathbf{\Sigma}) > 1.
\end{equation}
\end{lemma}

Several previous works \cite{jasf20, jaka17, wuws22, liwa22} empirically demonstrate and use an alignment assumption between $\mathbf{J}^*$ and $\mathbf{\Sigma}$. We make a similar assumption as given below:

\begin{assumption}\label{assumption:alignment}
$\mathbf{J}^* = \beta \mathbf{\Sigma}$, where $\beta>0$ is a constant that depends on hyper-parameters, $B$ and $n$.
\end{assumption}

\citet{jasf20} also point out that $\beta$ depends on $B$ and the number of training examples only. Given assumption 3.1, the conditions for stability, instability and edge (boundary) can be written more simply as:
\begin{eqnarray}
    & \mbox{Stability Condition:  } \alpha \lambda_1(\mathbf{J}^*) \le 1 \label{eq:newstab} \\
    & \mbox{Instability Condition:  } \alpha \lambda_1(\mathbf{J}^*) > 1 \label{eq:newinstab} \\
    & \mbox{Edge of Stability:  } \alpha \lambda_1(\mathbf{J}^*) = 1 \label{eq:newedge} 
\end{eqnarray}
where $\alpha = \max\{\frac{\eta}{2}, \eta^2\beta \}$.

\def\Hbar{\bar{\mathbf{H}}}

\subsection{Application to SGD, SAM and mSAM}
\label{subsec:stmethods}

We delve into applying the findings from section~\ref{subsec:linstab} to SGD, SAM and mSAM. For differentiation, we employ subscripts 1, 2, and 3 to denote SGD, SAM, and mSAM. 
Since $\mathbf{w}^*$ is a minimum, we assume that, for each example $i$, the Hessian, $\mathbf{H}_i(\mathbf{w}^*)$ is symmetric and positive semi-definite (PSD). Given this foundation, all other matrices we engage with in the subsequent analysis also turn out to be PSD. For SGD, \citet{wucw18} show that:  
\begin{equation} \label{eq:jssgd}
    \mathbf{J}_1^* = \Hbar^* \mbox{ and } \mathbf{\Sigma}_1 = \E [(\mathbf{H}_{\mathcal{S}}^* - \Hbar^*)^2] = \mathbf{V}(\mathbf{H}_{\mathcal{S}}^*),
\end{equation}
where $\mathbf{H}_{\mathcal{S}}^*$ is the Hessian of $\mathcal{L}_{\mathcal{S}}$ at $\mathbf{w}^*$, $\Hbar^*$ is the Hessian of the full batch mean loss at $\mathbf{w}^*$ and $\mathbf{V}(\mathbf{X}_{\mathcal{S}}) = \E[(\mathbf{X}_{\mathcal{S}} - \E[\mathbf{X}_{\mathcal{S}}])^2]$ denotes the variance of $\mathbf{X}_{\mathcal{S}}$. 

\begin{assumption}
For analysis with SAM and mSAM, we leave out the normalization terms in (\ref{sam-gradient}) and (\ref{msam-gradient}).    
\end{assumption}
In a recent study, \citet{agda23} analyze the stability characteristics of the full batch unnormalized SAM, as described in (\ref{sam-gradient}), where $\mathcal{S}$ signifies the complete training set. Notably, they too exclude the normalizer term $\|\nabla\mathcal{L}_{\mathcal{S}}(\mathbf{w})\|_2$ to simplify the analysis. However, unlike \citet{agda23}, we do not require the full batch assumption. 

Let us now apply the general linear stability analysis that we develop in section~\ref{subsec:linstab} to SAM and mSAM; for them, $\mathbf{d}_{\mathcal{S}}$ is given by
\begin{equation}\label{dsam1}
\mathbf{d}_{2,\mathcal{S}}(\mathbf{w}) = \mathbf{d}_{1,\mathcal{S}}(\mathbf{w} + \rho \mathbf{d}_{1,\mathcal{S}}(\mathbf{w})),
\end{equation}
\begin{equation}\label{dmSAM1}
\mathbf{d}_{3,\mathcal{S}}(\mathbf{w}) = \frac{1}{m} \sum_{j=1}^m \mathbf{d}_{2,\mathcal{S}_j}(\mathbf{w}).
\end{equation}
where $\mathbf{d}_{1,\mathcal{S}}$ is the SGD gradient given by
\begin{equation}\label{dsgd}
\mathbf{d}_{1,\mathcal{S}}(\mathbf{w}) = \frac{1}{|\mathcal{S}|} \sum_{i\in \mathcal{S}} \nabla\mathcal{L}(\mathbf{w}; e_i),
\end{equation}
and $e_i=(\mathbf{x}_i,y_i)$ denotes the $i$-th training example.
It is clear from (\ref{dsam1}) and (\ref{dmSAM1}) that the linear approximations of SAM and mSAM, i.e., the determination of $\mathbf{J}_{1, \mathcal{S}}(\mathbf{w}^*)$ and $\mathbf{J}_{2, \mathcal{S}}(\mathbf{w}^*)$ depend on the linearization of $\nabla \mathbf{d}_{1,{\mathcal{S}}}$ which is given by:
\begin{equation}
    \nabla \mathbf{d}_{1,{\mathcal{S}}} = \mathbf{H}_{\mathcal{S}}^* \mathbf{w} \Rightarrow \mathbf{J}_{1,\mathcal{S}}(\mathbf{w}^*) = \mathbf{H}_{\mathcal{S}}^* \defeq \frac{1}{|\mathcal{S}|} \sum_{i\in \mathcal{S}} \mathbf{H}_i(\mathbf{w}^*)
\end{equation}
When the details are worked out (see appendix~\ref{sec:proofs3}) we obtain expressions for $\mathbf{J}_{1, \mathcal{S}}(\mathbf{w}^*)$ and $\mathbf{J}_{2, \mathcal{S}}(\mathbf{w}^*)$ given by the following lemma.

\begin{lemma}
(a) $\mathbf{J}^*$ for SAM and mSAM are given by
\begin{eqnarray}
    & \mathbf{J}_2^* = \Hbar^* + \rho (\Hbar^*)^2 + \rho\mathbf{\Sigma}_1, \label{eq:J2} \\
    & \mathbf{J}_3^* = \Hbar^* + \rho (\Hbar^*)^2 + \rho\mathbf{\Sigma}_1 + \mathbf{\Omega}, \label{eq:J3}
\end{eqnarray}
where 
\begin{equation} \label{eq:gamom}
    \mathbf{\Omega} = \E\left[\frac{\rho}{m} \sum_{j=1}^m (\mathbf{H}_{\mathcal{S}_j}^* - \mathbf{H}_{\mathcal{S}}^*)^2\right].
\end{equation}
(b) $\mathbf{\Sigma}$ for SAM and mSAM are given by
\begin{eqnarray}
    & \mathbf{\Sigma}_2 = \E[(\mathbf{H}_{\mathcal{S}}^* + \rho (\mathbf{H}_{\mathcal{S}}^*)^2 - \mathbf{J}_2^*)^2] \label{eq:sig2} \\
    & \mathbf{\Sigma}_3 = \E [(\mathbf{J}_{3,\mathcal{S}}(\mathbf{w}^*) - \mathbf{J}_3^*)^2] \label{eq:sig3}
\end{eqnarray}
where
\begin{eqnarray*}
    & \mathbf{J}_{3,\mathcal{S}}(\mathbf{w}^*) = \mathbf{H}_{\mathcal{S}}^* + \rho (\mathbf{H}_{\mathcal{S}}^*)^2 + \frac{\rho}{m} \sum_{j=1}^m (\mathbf{H}_{\mathcal{S}_j}^* - \mathbf{H}_{\mathcal{S}}^*)^2. \label{eq:x3}
\end{eqnarray*}
(c) Equivalently, $\mathbf{\Sigma}_2$ and $\mathbf{\Sigma}_3$ can also be written as matrix variances:
\begin{eqnarray}
    & \mathbf{\Sigma}_2 = \mathbf{V}(\mathbf{H}_{\mathcal{S}}^* + \rho (\mathbf{H}_{\mathcal{S}}^*)^2) \label{eq:sigvar2} \\
    & \mathbf{\Sigma}_3 = \mathbf{V}(\mathbf{J}_{3,\mathcal{S}}(\mathbf{w}^*)) \label{eq:sigvar3}
\end{eqnarray}    
\end{lemma}

\begin{table*}[htbp]
\begin{sc}
\begin{center}
\begin{tabular}{ccccc}
\toprule
Dataset & Model & Vanilla & SAM &  mSAM 
\\ \midrule
\multirow{3}{*}{CIFAR 10} & ResNet50     & $95.45\pm 0.10$  & $96.09\pm0.11$ &  $96.40\pm0.06$\\ 
& WRN-28-10     & $95.92\pm0.12$ & $96.90\pm0.05$ &  $96.95\pm0.04$ \\
& ViT-B/16     & $97.68\pm0.03$  & $97.75\pm0.06$ &  $98.29\pm0.08$ \\\hline
\multirow{3}{*}{CIFAR 100} & ResNet50      & $80.68\pm0.13$ & $81.49\pm0.18$ & $83.37\pm0.10$\\
& WRN-28-10     &$81.01\pm 0.19$  & $82.93\pm0.13$ &  $84.07\pm0.06$\\
 & ViT-B/16      & $88.02\pm0.17$ & $88.75\pm0.07$ &   $89.00\pm 0.17$\\\hline

\multirow{4}{*}{ImageNet 1k} & ResNet50   & $76.35\pm0.09$ & $76.61\pm0.07$ & $76.79\pm0.09$ \\   
 & WRN50-2-bottleneck   & 78.04$\pm0.09$ & $78.50\pm0.09$ & $79.36\pm0.04$ \\   
 & ViT-S/32   & $64.53\pm0.72$ & $65.86\pm0.81$ & $66.76\pm0.22$ \\   
 \bottomrule

\end{tabular}
\end{center}
\end{sc}
\caption{Accuracy Results for CNN Architectures}
\label{cifar-table}
\end{table*}

It is imperative to acknowledge that the theoretical analysis of SAM-like methodologies with the normalization term presents a formidable undertaking, as highlighted by \citet{daas23}. Consequently, we defer this intricate endeavour to a subsequent phase.

\subsection{Implications on Sharpness}
It is apparent from the expressions in (\ref{eq:jssgd}), (\ref{eq:sigvar2}), and (\ref{eq:sigvar3}) that, in the progression from SGD to mSAM, the $\mathbf{\Sigma}$'s are computed as variances of matrices, with the inclusion of supplementary stochastic matrices at each stage. Consequently, one can anticipate the validity of the following theorem.

\begin{theorem}
\label{stability-sam-like-methods}
For any given $\eta$, $B$ and $\mathbf{w}^*$, (a) if SGD is unstable then SAM is unstable; and, (b) if SAM is unstable, then mSAM is unstable.     
\end{theorem}

Thanks to the rigorous theoretical and empirical analyses by \citet{cokl21, cogk22, arlp22}, GD on neural network training is understood to reach and operate in a regime known as the \textit{Edge of Stability (EoS)}, where the maximum eigenvalue of the training loss Hessian hovers just above the value $2/\eta$ (see Lemma \ref{eos-sgd} above). Upon entering the \textit{EoS} region, the training loss exhibits non-monotonic behavior over brief time intervals, while consistently decreasing over longer periods. In the mini-batch setting, \citet{jasf20} conduct systematic experiments, applying the stability theory introduced by \citet{wucw18}, thereby identifying a corresponding \textit{EoS} behavior for SGD.

The assertion of Theorem \ref{stability-sam-like-methods} remains independent of the reliance on assumption~\ref{assumption:alignment}. Through the meticulous delineation of precise mathematical expressions, this theorem unveils the inherent propensity of mSAM to manifest greater susceptibility to instability in contrast to SAM, attributed to the incorporation of micro-batching. It is notable that SAM, owing to its adversarial step, exhibits a heightened level of instability as compared to SGD. For the mini-batch setting, these results are new. Combined with the way $\Hbar^*$ is involved in the expressions, theorem \ref{stability-sam-like-methods}, assumption~\ref{assumption:alignment}, and the EoS theory culminate in the following pivotal finding, which serves as a direct comparative analysis of the sharpness inherent in the three methods.

\begin{theorem}
\label{theorem:lambdaHess}
For any given $\eta$ and $B$, if $\mathbf{H}_1^*$, $\mathbf{H}_2^*$ and $\mathbf{H}_3^*$ denote, respectively the Hessians of SGD, SAM and mSAM at their edge of stability, then $\lambda_1(\mathbf{H}_1^*) \ge \lambda_1(\mathbf{H}_2^*) \ge \lambda_1(\mathbf{H}_3^*)$.
\end{theorem}

Appendix~\ref{sec:proofs3} gives a proof of this result. Let us give a rough explanation of this important result. Comparing the expressions for $\mathbf{J}_1^*$, $\mathbf{J}_2^*$ and $\mathbf{J}_3^*$ in (\ref{eq:jssgd}) (\ref{eq:J2}) and (\ref{eq:J3}), we can see from the additional terms that come as we go from $\mathbf{J}_1^*$ to $\mathbf{J}_2^*$ to $\mathbf{J}_3^*$ that, for meeting the edge of stability condition $\lambda_1(\mathbf{J}_i^*)=1 \; \forall i$, $\lambda_1(\mathbf{H}_3^*)$ cannot be larger than $\lambda_1(\mathbf{H}_2^*)$, which in turn cannot be larger than $\lambda_1(\mathbf{H}_1^*)$. Thus, mSAM has better flatness than SAM, and SAM has better flatness than SGD.

We support Theorem \ref{theorem:lambdaHess} with an empirical investigation on three datasets in Section~\ref{sec:deeper}; see Table~\ref{eigen-table} there and the discussion below it.

\section{Numerical Experiments}
\label{sec:expts}

This section compares mSAM to SAM and vanilla optimization methods (i.e. without sharpness-aware modification) on various model architectures and datasets. We report the average and standard deviation of accuracy on the test data over five independent runs. We also note that we use the same values of hyper-parameters for all algorithms, where $\rho$ is chosen based on the best validation performance for SAM, and other hyper-parameters are chosen based on the best validation error for vanilla methods. Moreover, although using different values of $\rho$ for each micro-batch in mSAM is possible, doing so requires tuning numerous hyper-parameters, which is computationally infeasible. Therefore, we use the exact value of $\rho$ for all micro-batches.

\subsection{Image Classification}\label{imagesec}
In our first set of experiments on image classification datasets,
we compare the performance of mSAM, SAM and vanilla methods with multiple CNN architectures such as ResNets~\citep{resnets} and WideResNet~\citep{wideresnet}. We use CIFAR10/100~\citep{cifar} and ImageNet~\citep{imagenet} datasets as our test bed. We use different seeds for all five runs of each experiment.

The average accuracies corresponding to the experiments with the three datasets and the architectures of ResNet and WideResNet are reported in Table~\ref{cifar-table}. For the CIFAR datasets, we use an effective batch size of $512$ across four NVIDIA V100 GPUs. For the ImageNet dataset, we use an effective batch size of $2048$ across eight GPUs. We use $m=32$ for mSAM for CIFAR and $4$/$8$ for ResNet50 and WideResNet-50-2, respectively, for ImageNet. Details about hyper-parameters used to produce these results can be found in Appendix~\ref{imageparam}. Overall, mSAM consistently leads to better accuracy than SAM and vanilla methods in all CNN-related experimental results reported in Table~\ref{cifar-table}. 

Following recent results that suggest that sharpness-aware optimization can substantially improve the generalization quality of Vision Transformers (ViTs) ~\citep{samvit}, we conduct additional experiments on ViT architectures. In particular, we use the pre-trained ViT-B/16 checkpoint from~\citep{vitpaper} and fine-tune the model on CIFAR10/100 data independently. For ImageNet, we train a smaller version of the ViT model (ViT-S/32) from scratch.  We choose $512$ as the batch size for the CIFAR fine-tuning tasks and use $4096$ as the batch size for training from scratch on the ImageNet dataset. The average accuracy results for ViT are reported in Table~\ref{cifar-table}. Similar to results on CNNs, mSAM outperforms both SAM and vanilla training across all sets of ViT-related tasks. Note that we do not leverage any advanced data augmentation techniques, and only use inception-style image pre-processing.  Other hyper-parameters to produce these results are listed in Appendix~\ref{imageparam}. 

\subsection{NLP Fine-tuning}
Our next set of experiments is based on four tasks from the GLUE benchmark~\citep{glue}. In particular, we choose COLA and MRPC as two small datasets and SST-2 and QQP as two larger datasets for empirical evaluation. Fine-tuning experiments are performed with the RoBERTa-base model~\citep{liu2019roberta} on four NVIDIA V100 GPUs with an effective batch size of $32$. For the ease of reproduction of the results, we tabulate all the hyper-parameters used in Appendix~\ref{glueparam}. For the fine-tuning experiments, we report the average value of Matthews Correlation Coefficient for COLA, and average accuracy for other datasets in Table~\ref{nlp-table}. Overall, mSAM performs better than the baseline methods on these datasets. However, the variance among different runs is comparably high for smaller datasets such as COLA and MRPC. On the other hand, the results on larger data such as SST-2 and QQP are expectedly more robust across different runs.

\begin{table*}[htbp]
    \begin{minipage}{.49\textwidth}
      \centering
        \begin{tabular}{cccc}
        \toprule
        Task &Vanilla & SAM & mSAM ($m=8$)
        \\ \midrule
        COLA        & $63.66\pm2.46$ & $64.30\pm0.49$ & $64.57\pm 0.66$\\
        MRPC        & $89.79\pm0.05$ & $90.37\pm0.13$ & $90.92\pm0.16$ \\
        SST-2        & $94.27\pm0.18$ & $95.21\pm0.12$  & $95.38\pm0.10$  \\
        QQP        & $91.70\pm0.11$ & $92.13\pm0.02$ &  $92.18\pm0.03$ \\
        \bottomrule
        \end{tabular}
      \caption{Accuracy Results for GLUE Tasks}
      \label{nlp-table}
    \end{minipage}%
    \begin{minipage}{.49\textwidth}
      \centering
        \begin{tabular}{cccc}
        \toprule
         Model & Vanilla & SAM &  mSAM
        \\  
        \midrule
        ResNet50     &  $26\pm2$  & $21\pm3$  & $18\pm1$ \\
        WRN-28-10     &   $92\pm4$ & $30\pm2$  & $17\pm1$ \\
        \bottomrule
        \end{tabular}
      \caption{$\lambda_{\max}$ (sharpness) for CNNs}
      \label{eigen-table}
    \end{minipage}
\end{table*}

\section{A Deeper Investigation of mSAM}\label{sec:deeper}
To further understand the mSAM algorithm, we design and report some experiments in this section. Additional experimental results are moved to Appendix~\ref{app:switch}.

\begin{figure}[htbp]
     \centering
\begin{tabular}{cc}
 ResNet50 & WRN-28-10 \\
     \includegraphics[trim={0.93cm 0.5cm 0.93cm 0.5cm},clip,width=0.47\linewidth]{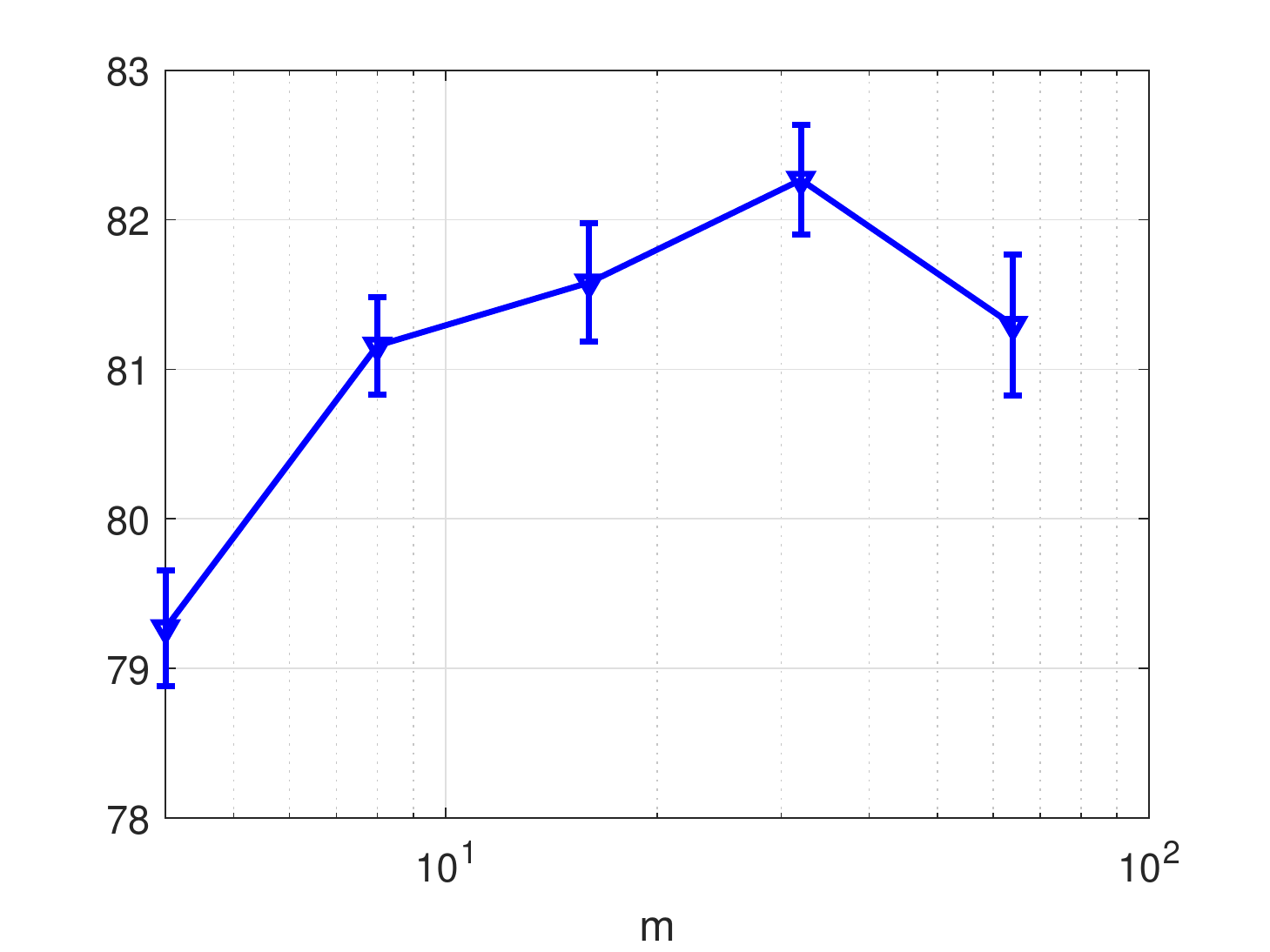}&   
     \includegraphics[trim={0.93cm 0.5cm 0.93cm .5cm},clip,width=0.47\linewidth]{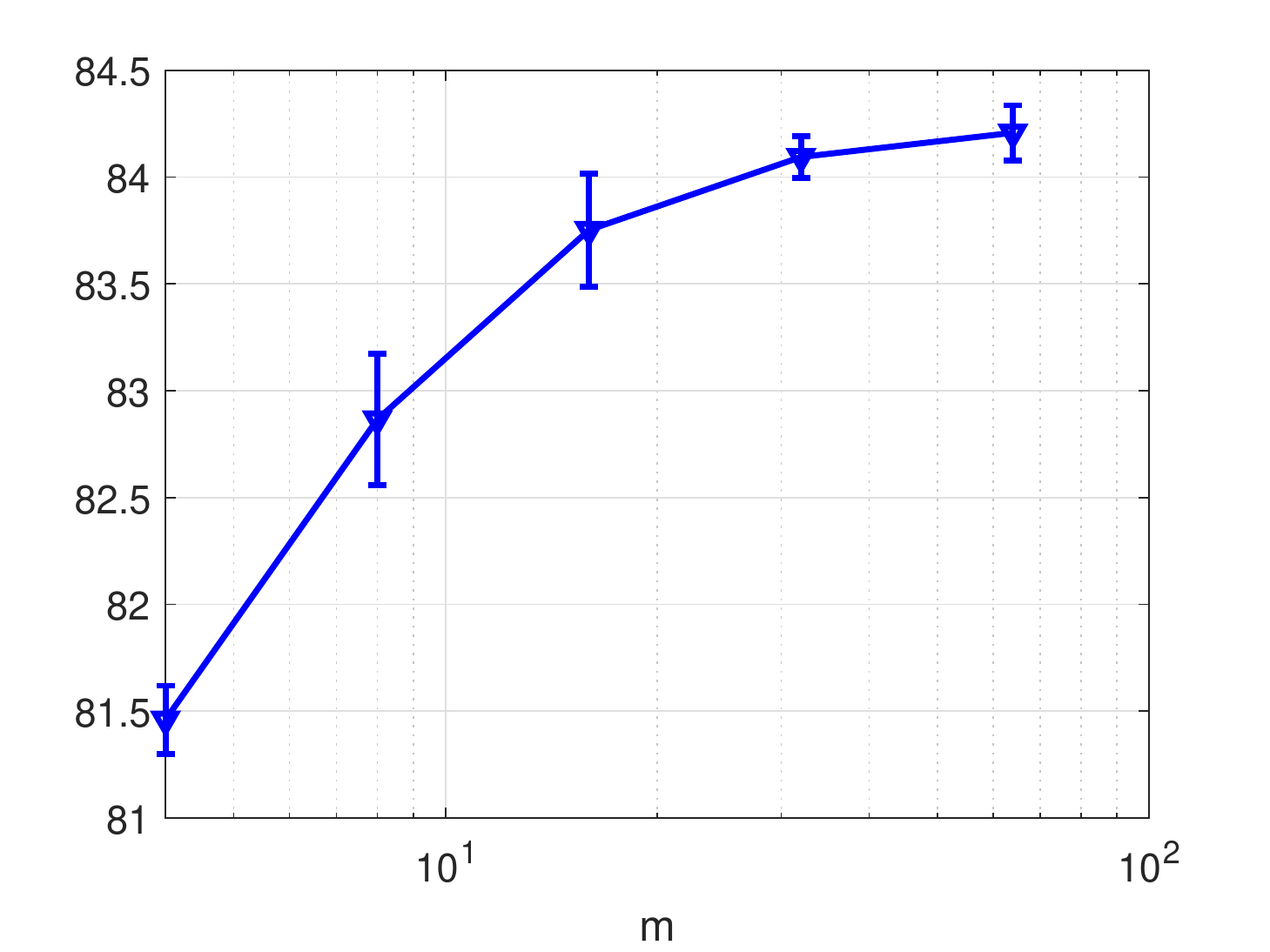}  \\
     $m$ & $m$
\end{tabular}
        \caption{The effect of varying $m$ on accuracy. We see that increasing $m$ up to 32 results in better accuracy. However, increasing $m$ further leads to worse results/marginal improvements.}
        \label{fig:mb}
\end{figure}

\paragraph{Effect of varying $m$:}
In our experiments, we have observed that a larger value of $m$ often leads to better test accuracy. We recover SAM by setting $m=1$, which produces inferior results. To test this hypothesis, we set up the experiments with the CIFAR100 dataset on two CNNs, ResNet50 and WRN-28-10, in the same setup as in Section~\ref{imagesec}. We run mSAM for different values of $m\in\{4,8,16,32,64\}$. The accuracy results for these experiments are shown in Figure~\ref{fig:mb}.

Increasing $m$ improves the performance up to $m\approx 32$. However, a value of $m$ larger than this threshold either leads to worse performance or marginal improvements, so increasing $m$ does not necessarily result in better generalization. Intuitively, when the micro-batch is too small, the perturbation derived according to the micro-batch might not be a good estimate of the actual SAM perturbation, leading to worse performance. We leave the theoretical analysis of such a phenomenon an interesting direction for future research. We also note that understanding how the optimal value of $m$ and batch size interact is an open question for future work.

\paragraph{Are mSAM solutions flat?}
The SAM algorithm hypothesizes that flat solutions generalize better. Since mSAM consistently outperforms SAM, it is worth investigating if mSAM settles for even flatter solutions than SAM, as predicted by our theory in section~\ref{sec:theory}. To that end and to quantify sharpness, we calculate the largest eigenvalue of the Hessian of the loss function $\mathcal{L}_{\mathcal{S}}(\Vec{w})$ at the final solution, denoted as $\lambda_{\max}$. We calculate $\lambda_{\max}$ over the full train data using power iteration, as implemented by~\cite{hessian-eigenthings}. We use the ResNet50 and WRN-28-10 models trained on CIFAR100 (see Section~\ref{imagesec}) to calculate $\lambda_{\max}$. The average results for these experiments are reported in Table~\ref{eigen-table}. We see that mSAM leads to solutions with smaller $\lambda_{\max}$ than SAM and vanilla SGD, and this finding agrees with Theorem~\ref{theorem:lambdaHess}.

\begin{table}[ht!]
\small
\begin{sc}
\begin{center}
\begin{tabular}{ccccc}
\toprule
 Model & Vanilla & SAM &  mSAM 
\\  
\midrule
ResNet50      & $14.74$ & $25.54$  & $25.87$ \\
 WRN-50-2-bottleneck     & $27.68$ & $52.02$  & $58.93$ \\
 ViT-S/32       & $19.91$ & $24.50$  & $30.66$ \\
 \bottomrule
\end{tabular}
\end{center}
\end{sc}
\caption{Runtime of different methods and architectures on ImageNet data (in hours)}
\label{runtime-table-imagenet-main}
\end{table}

\paragraph{mSAM runtime:}\label{app:runtime}
A general misconception about mSAM is that it is computationally inefficient, as the total number of forward-backwards passes in the network is multiplied by $m$~\cite{nonsensemsam}. However, note that these passes are performed on micro-batches, which are $m$ times smaller than the actual minibatch. Hence, the overall computational cost gets amortized and is never as high as $m$ times the cost of SAM. In practice, on large networks, the runtime of mSAM is \emph{only} $1.1$-$1.2$ times more compared to SAM. Particularly, in Table~\ref{runtime-table-imagenet-main} we report the runtime of SGD, SAM and mSAM for the ImageNet data. The overhead of mSAM over SAM for ViT-S/32 is about $20\%$, while for ResNet50 the overhead is negligible. We discuss the computational efficiency of mSAM in more details in Appendix~\ref{app:switch}, as well as discussing a few hybrid algorithms to reduce the computational cost of mSAM even further.

\section{Discussion}
Within the confines of this study, we proffer a theoretical rationale that explains how mSAM leads to flatter solutions when compared with SAM. This perspective extends a contemporary framework of stability dynamics. It is discernible from intuitive observations that minima characterized by heightened flatness frequently correlate with enhanced generalization capabilities. Our comprehensive empirical inquiry reinforces such theoretical assertion by demonstrating the superiority of mSAM over SAM across diverse datasets with image classification and NLP tasks and a spectrum of model architectures, including CNNs and Transformers. The extent of performance differentiation is notably contingent upon the specifics of the dataset and the inherent architecture. Our empirical endeavours reveal that the computational overhead associated with mSAM does not incur a significantly higher cost than SAM, thereby establishing its viability for solving large-scale problems.

Furthermore, our theoretical framework is amenable to broader generalization, specifically, the dissection of the conventional mini-batch into micro-batches. This extension becomes particularly pertinent when considering gradient computation of a compositional nature, an attribute often observed in other iterations of sharpness-aware minimization algorithm. An avenue ripe for future exploration entails the augmentation and extension of the flatness theory as applied to the segmentation of mini-batches into micro-batches, particularly in the context of gradient updates characterized by non-linear aggregation. Such a development presents an intriguing prospect wherein this technique could be pragmatically harnessed to enhance the generalization efficacy of alternative methodologies.

\section*{Acknowledgements}
Kayhan Behdin contributed to this work while he was an intern at LinkedIn during summer 2022 and 2023. This work is not a part of his MIT research. Rahul Mazumder contributed to this work while he was a consultant for LinkedIn (in compliance with MIT’s outside professional activities policies). This work is not a part of his MIT research.

\bibliographystyle{unsrtnat}
\bibliography{ref}

\appendix
\clearpage
\newpage
\appendix
\numberwithin{table}{section}
\numberwithin{figure}{section}

\section{Proofs of the results in section~\ref{sec:theory}}
\label{sec:proofs3}

\begin{lemma}
For a random symmetric PSD matrix $\mathbf{X}$:
 \begin{equation}
     \E[\mathbf{X}^2] = \E[\mathbf{X}]^2 + \E[(\mathbf{X} - \E[\mathbf{X}])^2].
 \end{equation}
\end{lemma}

\begin{proof}
Note that
\begin{align*}
    \E[(\mathbf{X} - \E[\mathbf{X}])^2] &= \E[\mathbf{X}^2 - \mathbf{X}\E[\mathbf{X}] - \E[\mathbf{X}]\mathbf{X} - \E(\mathbf{X})^2] \\ \nonumber
    &= \E [\mathbf{X}^2] - \E[\mathbf{X}]\E[\mathbf{X}] - \E[\mathbf{X}]\E[\mathbf{X}] - \E[\mathbf{X}]^2 = \E[\mathbf{X}^2] - \E[\mathbf{X}]^2 
\end{align*}
which proves the above lemma.
\end{proof}

\noindent\textbf{Proof of Lemma 3.1:} 
\begin{proof}
We have
\begin{equation}
    \E [(\mathbf{I} - \eta \mathbf{J}_{\mathcal{S}}(\mathbf{w}^*) )^2] = \E [ (\mathbf{I} - 2\eta \mathbf{J}_{\mathcal{S}}(\mathbf{w}^*) + \eta^2 (\mathbf{J}_{\mathcal{S}}(\mathbf{w}^*))^2 )] = (\mathbf{I} - 2\eta \mathbf{J}^* + \eta^2 (\mathbf{J}^*)^2) + \E [ (\mathbf{J}_{\mathcal{S}}(\mathbf{w}^*) - \mathbf{J}^*)^2 ]
\end{equation}
which completes the proof.    
\end{proof}

\noindent\textbf{Proof of Lemma 3.2:}
\begin{proof}
$S_1$ translates to the eigenvalues of $(\mathbf{I}-\mathbf{J}^*)$ to be between $-1$ and $1$. The upper bound holds automatically since $\mathbf{J}^*$ is PSD. The lower bound completes the proof.    
\end{proof}

\noindent{\bf Proof of Lemma 3.3:} This is just a restatement of Lemma 3.2.

\noindent{\bf Proof of Lemma 3.4:}
\begin{proof}
For SGD, SAM and mSAM, let us recall, from subsection~\ref{subsec:stmethods}, their expressions for $d$ given by
\begin{equation}\label{dsgd}
\mathbf{d}_{1,\mathcal{S}}(\mathbf{w}) = \frac{1}{|\mathcal{S}|} \sum_{e\in \mathcal{S}} \nabla\mathcal{L}(\mathbf{w}; e),
\end{equation}
\begin{equation}\label{dsam}
\mathbf{d}_{2,\mathcal{S}}(\mathbf{w}) = \mathbf{d}_{1,\mathcal{S}}(\mathbf{w} + \rho \mathbf{d}_{1,\mathcal{S}}(\mathbf{w})),
\end{equation}
\begin{equation}\label{dmSAM}
\mathbf{d}_{3,\mathcal{S}}(\mathbf{w}) = \frac{1}{m} \sum_{j=1}^m \mathbf{d}_{2,\mathcal{S}_j}(\mathbf{w}).
\end{equation}
Consider the linear approximations of the three methods to obtain their $\mathbf{J}_{\mathcal{S}}(\mathbf{w}^*)$. For SGD, we have:
\begin{equation}
    \nabla \mathbf{d}_{1,{\mathcal{S}}} = \mathbf{H}_{\mathcal{S}}^* \mathbf{w} \Rightarrow \mathbf{J}_{1,\mathcal{S}}(\mathbf{w}^*) = \mathbf{H}_{\mathcal{S}}^* \defeq \frac{1}{|\mathcal{S}|} \sum_{i\in \mathcal{S}} \mathbf{H}_i(\mathbf{w}^*)
\end{equation}
For SAM, we have:
\begin{equation}
 \nabla \mathbf{d}_{2,{\mathcal{S}}} = \mathbf{H}_{\mathcal{S}}^* (\mathbf{w} + \rho \mathbf{H}_{\mathcal{S}}^* \mathbf{w}) = [\mathbf{H}_{\mathcal{S}}^* + \rho (\mathbf{H}_{\mathcal{S}}^*)^2] \mathbf{w} \Rightarrow \mathbf{J}_{2,\mathcal{S}}(\mathbf{w}^*) = \mathbf{H}_{\mathcal{S}}^* + \rho (\mathbf{H}_{\mathcal{S}}^*)^2
\end{equation}
For mSAM, we have:
\begin{equation}
 \nabla \mathbf{d}_{3,{\mathcal{S}}} = [\frac{1}{m} \sum_{j=1}^m (\mathbf{H}_{\mathcal{S}_j}^* + \rho (\mathbf{H}_{\mathcal{S}_j}^*)^2)] \mathbf{w} = [\mathbf{H}_{\mathcal{S}}^* + \frac{\rho}{m} \sum_{j=1}^m (\mathbf{H}_{\mathcal{S}_j}^*)^2)] \mathbf{w} \Rightarrow \mathbf{J}_{3,\mathcal{S}}(\mathbf{w}^*) = \mathbf{H}_{\mathcal{S}}^* + \frac{\rho}{m} \sum_{j=1}^m (\mathbf{H}_{\mathcal{S}_j}^*)^2 \label{eq:J3S}
 \end{equation}
Applying Lemma A.1 to the second term of $\mathbf{J}_{3,\mathcal{S}}(\mathbf{w}^*)$ in Eq. \ref{eq:J3S} gives:
\begin{equation}\label{Jmsam}
\mathbf{J}_{3,\mathcal{S}}(\mathbf{w}^*) = \mathbf{H}_{\mathcal{S}}^* + \frac{\rho}{m} \sum_{j=1}^m (\mathbf{H}_{\mathcal{S}_j}^*)^2 = \mathbf{H}_{\mathcal{S}}^* + \rho (\mathbf{H}_{\mathcal{S}}^*)^2 + \frac{\rho}{m} \sum_{j=1}^m (\mathbf{H}_{\mathcal{S}_j}^* - \mathbf{H}_{\mathcal{S}}^*)^2
\end{equation}
Let us now derive $\mathbf{J}^*$ and $\mathbf{\Sigma}$ for the three methods. For SGD:
\begin{equation}
\mathbf{J}_1^* = \Hbar^*,  \; \mathbf{\Sigma}_1 = \E[(\mathbf{H}_{\mathcal{S}}^* - \Hbar^*)^2]
\end{equation}
SAM: Applying Lemma A.1 again gives
\begin{equation}
\mathbf{J}_2^* = \Hbar^* + \rho\E[(\mathbf{H}_{\mathcal{S}}^*)^2] = \Hbar^* + \rho(\Hbar^*)^2 + \rho\E[(\mathbf{H}_{\mathcal{S}}^* - \Hbar^*)^2] = \Hbar^* + \rho(\Hbar^*)^2 + \rho\mathbf{\Sigma}_1,  
\end{equation}
\begin{equation}
\mathbf{\Sigma}_2 = \E[(\mathbf{J}_{2,\mathcal{S}}(\mathbf{w}^*)-\mathbf{J}_2^*)^2] = \E[(\mathbf{H}_{\mathcal{S}}^* + \rho (\mathbf{H}_{\mathcal{S}}^*)^2 - \mathbf{J}_2^*)^2]
\end{equation}
mSAM:
\begin{equation}
    \mathbf{J}_3^* = \E[\mathbf{J}_{3,\mathcal{S}}(\mathbf{w}^*)] = \E[\mathbf{H}_{\mathcal{S}}^* + \rho (\mathbf{H}_{\mathcal{S}}^*)^2 + \frac{\rho}{m} \sum_{j=1}^m (\mathbf{H}_{\mathcal{S}_j}^* - \mathbf{H}_{\mathcal{S}}^*)^2]
\end{equation}
The expectations of the first two terms are exactly as in SAM. Thus,
\begin{equation}
    \mathbf{J}_3^* = \mathbf{J}_2^* + \mathbf{\Omega} = \Hbar^* + \rho(\Hbar^*)^2 + \rho\mathbf{\Sigma}_1 + \mathbf{\Omega}
\end{equation}
where $\mathbf{\Omega}$ is as in (\ref{eq:gamom}). Next,
\begin{equation}
    \mathbf{\Sigma}_3 = \E[(\mathbf{J}_{3,\mathcal{S}}(\mathbf{w}^*)-\mathbf{J}_3^*)^2]
\end{equation}
which is (\ref{eq:sig3}).
\end{proof}

\noindent{\bf Proof of Theorem 3.1} 
\begin{proof}
It follows in a straightforward way from (\ref{eq:newinstab}), (\ref{eq:jssgd}), (\ref{eq:J2}), and (\ref{eq:J3}), $\mathbf{J}_3^* \succeq \mathbf{J}_2^* \succeq \mathbf{J}_1^*$.
\end{proof}

\noindent{\bf Proof of Theorem 3.2}
\begin{proof}
Let $(\mathbf{J}_1^*, \mathbf{H}_1^*)$  $(\mathbf{J}_2^*, \mathbf{H}_2^*)$ and $(\mathbf{J}_3^*, \mathbf{H}_3^*)$ denote the linearized dynamics and Hessians of SGD, SAM and mSAM at their edges of stability, at points $\mathbf{w}_1^*$, $\mathbf{w}_2^*$, and $\mathbf{w}_3^*$ respectively.

Let us first compare SAM and mSAM. Let $\mathbb{R}_+$ denote the non-negative reals. Define the quadratic function with non-negative coefficients (and hence monotone), $f:\mathbb{R}_+\rightarrow \mathbb{R}_+$ as $f(x) = \alpha(1+\frac{\rho}{\beta}) x + \alpha\rho x^2$. 
For any symmetric PSD matrix, $\mathbf{A}$, $\lambda_1(f(\mathbf{A})) = f(\lambda_1(\mathbf{A}))$. This follows from the fact that, if $\mathbf{A}$ is symmetric PSD and has $\{\mbox{(eigenvalue, eigenvector)}\}$ pairs $\{(\mu_i, \mathbf{v}_i)\}$, then for any positive integer $p$, $\mathbf{A}^p$ has $ \{ (\mu_i^p, \mathbf{v}_i) \}$ as its set of (eigenvalue, eigenvector) pairs. Starting from the edge of stability condition for SAM (\ref{eq:newedge}), 
\begin{align*}
    \alpha \lambda_1(\mathbf{J}_2^*)
    &= \alpha \lambda_1(\mathbf{H}_2^* + \rho(\mathbf{H}_2^*)^2 + \rho \mathbf{\Sigma}_1) \\
    &= \lambda_1(\alpha\mathbf{H}_2^* + \alpha\rho(\mathbf{H}_2^*)^2 + \alpha\rho \mathbf{\Sigma}_1) \\
    &= \lambda_1(\alpha\mathbf{H}_2^* + \alpha\rho(\mathbf{H}_2^*)^2 + \frac{\alpha\rho}{\beta} \mathbf{H}_2) \\
    &= \lambda_1(f(\mathbf{H}_2^*)),
\end{align*}
where the third line follows from $\mathbf{H}_2^* = \EE[\mathbf{H}_{2,\mathcal{S}}^*] = \beta\mathbf{\Sigma}_1$, since the SGD dynamics around $\mathbf{w}_2^*$ satisfy assumption \ref{assumption:alignment}. Meanwhile, by the same steps we applied to the SAM dynamics while accounting for the extra term $\mathbf{\Omega}$, we find that:
\begin{gather*}
    \alpha \lambda_1(\mathbf{J}_3^*) = \lambda_1(f(\mathbf{H}_3^*) + \alpha \mathbf{\Omega}).
\end{gather*}
Since both the SAM and mSAM dynamics are assumed to satisfy the edge-of-stabililty condition, we have that $\lambda_1(\mathbf{J}_3^*) =  \lambda_1(\mathbf{J}_2^*)$. By Weyl's inequality,
\begin{gather}
    \lambda_1(f(\mathbf{H}_2^*))
    = \lambda_1(f(\mathbf{H}_3^*) + \alpha \mathbf{\Omega})
    \geq \lambda_1(f(\mathbf{H}_3^*)) + \alpha\lambda_{\min}(\mathbf{\Omega}).
\end{gather}
We know that $\mathbf{\Omega}$ is PSD, so we conclude that $\lambda_1(f(\mathbf{H}_2^*)) \geq \lambda_1(f(\mathbf{H}_3^*))$, hence $f(\lambda_1(\mathbf{H}_2^*)) \geq f(\lambda_1(\mathbf{H}_3^*))$. By the monotonicity of $f(\cdot)$, we conclude that $\lambda_1(\mathbf{H}_2^*) \geq \lambda_1(\mathbf{H}_3^*)$.

Now let us compare SGD and SAM. We have
\begin{equation}
    \lambda_1(\mathbf{H}_1^*) = 1 \mbox{ and } \lambda_1(\mathbf{H}_2^*) \le 1,
\end{equation}
where the second is an inequality because $\mathbf{J}_2^*$ in (\ref{eq:J2}) has an additional psd term, $\rho (\mathbf{H}_1^*)^2 + \rho\mathbf{\Sigma}_1$.
Now,
\begin{equation}
    \alpha \lambda_1(\mathbf{H}_2^*) \le 1 = \alpha \lambda_1(\mathbf{H}_1^*)
\end{equation}
and, since $\alpha>0$, we have $\lambda_1(\mathbf{H}_2^*) \le \lambda_1(\mathbf{H}_1^*)$, which completes the proof.
\end{proof}

\section{mSAM and Computational Efficiency}\label{app:switch}
In this section, we discuss details of mSAM implementation and review its computational efficiency. 
mSAM can be implemented either to have less memory footprint or to be faster. We choose to use the memory-efficient version, making mSAM more suitable for training larger models. Specifically, given a mini-batch of data such as $\mathcal{S}$, this mini-batch is divided into $m$ micro-batches in the system memory, and then load each micro-batch separately to the GPU memory whenever it is used. This leads to a slight runtime overhead due to the need to move the data (micro-batches) in and out of the GPU memory. However, we opted to use the memory-efficient implementation as it enables us to train models with any batch size, as long as the micro-batch size is sufficiently small. This choice of a memory-optimized implementation is due to the fact that newer DNN models tend to be larger. We reemphasize that although mSAM performs $m$-times many more forward-backward passes, each pass is done on a micro-batch that is $m$-times smaller. Therefore, in terms of forward-backward passes, SAM and mSAM are equivalent.

To be more specific, we report the runtime for SGD, SAM and mSAM in Table~\ref{runtime-table} for CIFAR100 data and in Table~\ref{runtime-table-imagenet} for ImageNet data. 

\begin{table}[ht!]
\caption{Runtime of different methods and architectures on CIFAR100 data (in seconds)}
\label{runtime-table}
\begin{footnotesize}
\begin{sc}
\begin{center}
\begin{tabular}{ccccc}
\toprule
 Model & Vanilla & SAM &  mSAM 
\\  
\midrule
ResNet50      & $4497\pm11$ & $7440\pm9$  &  $16196\pm77$ \\
 WRN     & $10675\pm18$  & $17483\pm40$  &   $22261\pm24$\\
 ViT-B/16      & $4349\pm21$ & $7007\pm43$ &   $8163\pm14$ \\
 \bottomrule
\end{tabular}
\end{center}
\end{sc}
\end{footnotesize}
\end{table}

\begin{table}[ht!]
\caption{Runtime of different methods and architectures on ImageNet data (in hours)}
\label{runtime-table-imagenet}
\begin{footnotesize}
\begin{sc}
\begin{center}
\begin{tabular}{ccccc}
\toprule
 Model & Vanilla & SAM &  mSAM 
\\  
\midrule
ResNet50      & $14.74\pm0.68$ & $25.54\pm0.25$  & $25.87\pm0.08$ \\
 WRN-50-2-bottleneck     & $27.68\pm0.18$ & $52.02\pm0.36$  & $58.93\pm0.26$ \\
 ViT-S/32       & $19.91\pm0.13$ & $24.50\pm0.26$  & $30.66\pm0.15$ \\
 \bottomrule
\end{tabular}
\end{center}
\end{sc}
\end{footnotesize}
\end{table}

Since SAM requires two forward-backwards passes for each batch of data, SAM is almost twice as slow as vanilla training. In our experiments, mSAM appears to be slower than SAM, although not $m$ times slower, as suggested by~\citet{nonsensemsam}. Expectedly, SAM is almost twice as slow as the vanilla method in most cases (despite ViT-S/32 experiment where the common data pre-processing stage requires more time for preparing the images into a sequence of 32 patches). We see that in the worst case, mSAM is \emph{only} twice as slow as SAM, and in the best case, the computational penalty is only within $10\%$ increase compared to SAM. Interestingly, for large models the runtime overhead of mSAM seems insignificant. For example in Table~\ref{runtime-table} for CIFAR100, the overhead is the smallest for ViT with 86M parameters, then WRN with 36M parameters has the best performance, and then ResNet50 with 23M parameters. This can be explained as we noted above, the mSAM overhead results from loading micro-batches to GPU. Note that this data communication overhead is constant regardless of the model size. This leads to the runtime overhead of mSAM being smaller for larger models, where more time is spent in the forward-backward pass.

\begin{figure}[htbp]
\centering

\begin{tabular}{cc}
 ResNet50 & WRN-28-10 \\
     \includegraphics[scale=0.50]{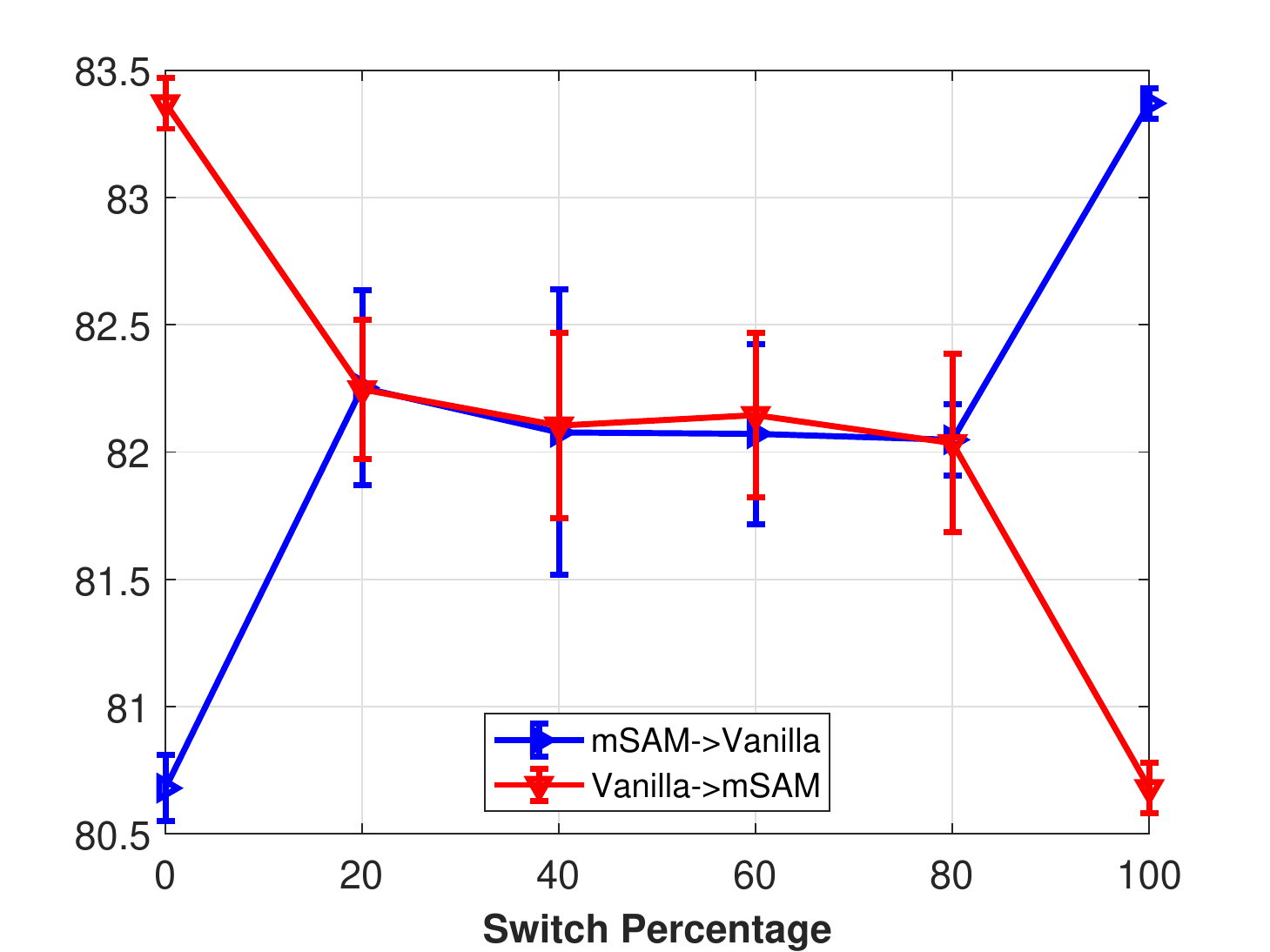}&   
     \includegraphics[scale=0.50]{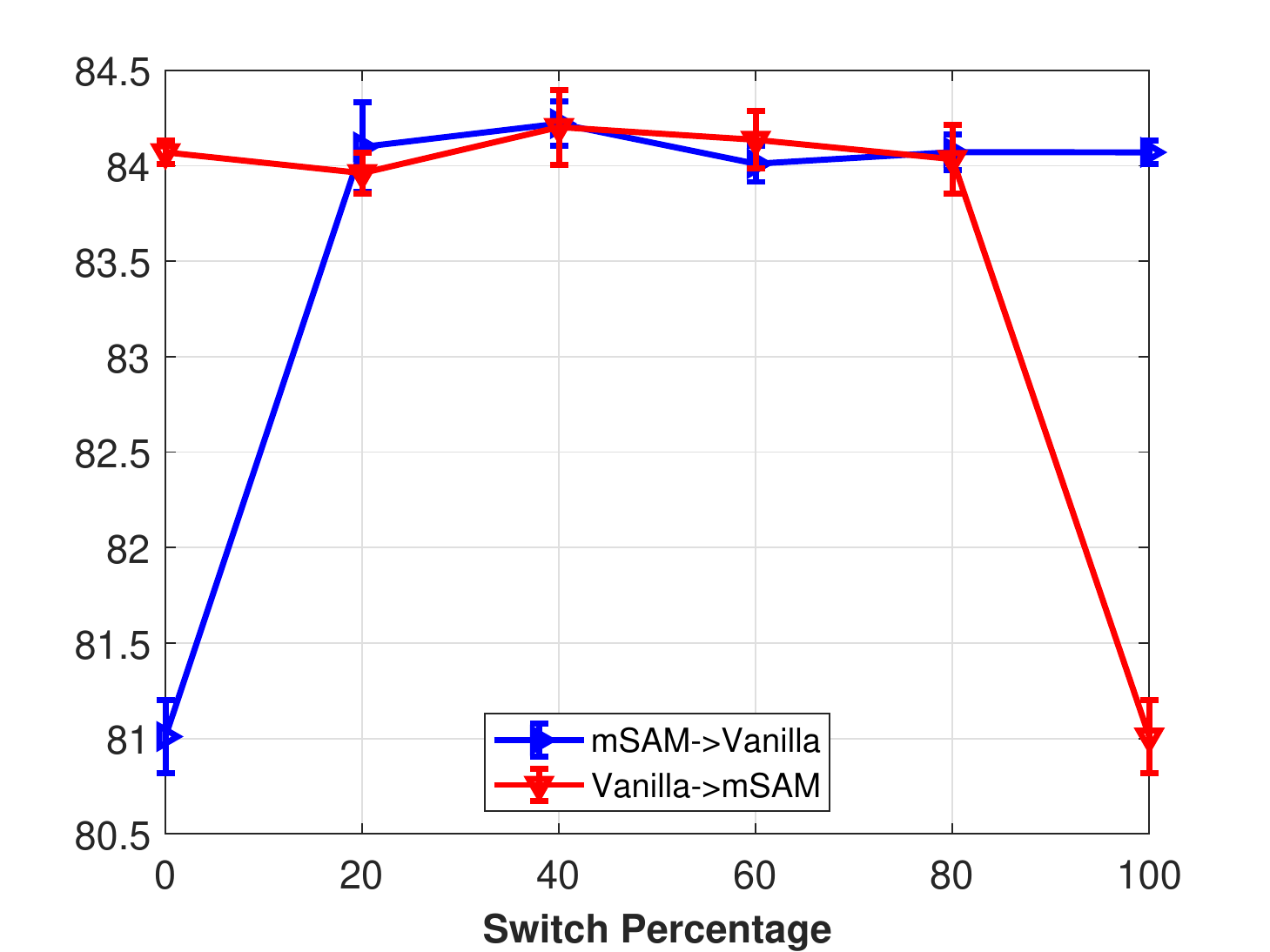}  
\end{tabular}
\caption{Effect of switching training algorithm}
\label{fig:switch}
\end{figure}

Although mSAM does not appear to be computationally prohibitive in our experiments, it is still not as efficient as vanilla training, leaving room for further improving its efficiency. To that end, we conduct the following set of experiments. Building on our CIFAR100 experiments from Section~\ref{imagesec}, we start the training either with mSAM or vanilla training and then switch to the other training algorithm at some point. We keep all other training parameters fixed. The accuracy results for this setup for ResNet50 and WRN-28-10 are reported in Figure~\ref{fig:switch}. In this figure, the switch percent is the threshold in training when we transition from one algorithm to the other. For example, for the switch percent of $20$, if we start with mSAM, we use mSAM for the first $20\%$ of epochs and vanilla updates for the rest. If mSAM is used for the initial and/or final part of training, the accuracy is always better than vanilla training. In fact, in the WRN-28-10 case, as long as we partially use mSAM, the accuracy is almost the same as training with mSAM for the entire duration. For ResNet50, not using mSAM for the whole training leads to a drop in performance; however, even in this case, the accuracy of the hybrid training is better than the SAM training. These observations suggest that it is possible to enjoy the superior performance of mSAM, at least to some degree, while not having to deal with the computational complexity of mSAM for the entire training. A better theoretical and empirical understanding of the hybrid training method can be an exciting avenue for future work.

\section{Hyper-parameters for Image Classification Experiments}\label{imageparam}
As mentioned, our experiments on CIFAR data in this section are done on 4 Nvidia V100 GPUs, with an effective batch size of 512. For mSAM, we used the micro-batch size of $16$, corresponding to $m=32$. The rest of the hyper-parameters are chosen as in Table~\ref{cifar-params} for CIFAR10/100 experiments. For ImageNet experiments, we use 8 Nvidia V100 GPUs, with an effective batch size of 2048 for ResNet50 and WideResNet-50-2-bottleneck, and 4096 for ViT-S/32.   The micro-batch size and $m$ are chosen based on a grid search within the set of $\{2, 4, 8, 16\}$. The rest of the hyper-parameters are listed in Table~\ref{imagenet-params} for ImageNet experiments. 

\begin{table}[t]
\caption{Hyper-parameters for CIFAR10/100 experiments}
\label{cifar-params}
\begin{footnotesize}
\begin{sc}
\begin{center}
\begin{tabular}{cccc}
\toprule
Model & ResNet50 & WRN-28-10 & ViT-B/16 (fine-tuning) \\
\midrule
Optimizer & SGD & SGD & AdamW\\
Peak Learning Rate & 0.5 & 0.75  &  $10^{-3}$ \\
Batch Size &512 &512 &512\\
Number of epochs &200 &200 &20\\
Momentum &0.9 & 0.9  & - \\
Weight Decay &$5\times 10^{-4}$ &$5\times 10^{-4}$ &0.3\\
Label Smoothing &0.1 &0.1 & - \\
Learning Rate Schedule &\multicolumn{3}{c}{All methods use one cycle with $5\%$ warm-up}\\
Gradient Clipping &-  &- & norm=1\\
$\rho$ (SAM/mSAM) &0.2 &0.2  &0.3\\
$m$ (mSAM) & 32 & 32 & 32 \\
\bottomrule
\end{tabular}
\end{center}
\end{sc}
\end{footnotesize}
\end{table}

\begin{table}[t]
\caption{Hyper-parameters for ImageNet experiments}
\label{imagenet-params}
\begin{footnotesize}
\begin{sc}
\begin{center}
\begin{tabular}{cccc}
\toprule
Model & ResNet50 & WRN-50-2-bottleneck & ViT-S/32  \\
\midrule
Optimizer & SGD & SGD & AdamW \\
Peak Learning Rate & 0.8 & 0.5 & 3e-3  \\
Batch Size &2048 &2048 &4096\\
Number of epochs &90 &100 &300 \\
Momentum &0.9 & 0.9  & -  \\
Weight Decay &$1\times 10^{-4}$ &$1\times 10^{-4}$ & 0.3  \\
Label Smoothing &0.1 &0.1 & -  \\
Learning Rate Schedule & 17.7\% (5K step) & 1\% & 10\% \\
Gradient Clipping &-  &- & norm=1  \\
$\rho$ (SAM/mSAM) &0.05 &0.05  &0.05\\
$m$ (mSAM) & 4 & 8 & 8\\
\bottomrule
\end{tabular}
\end{center}
\end{sc}
\end{footnotesize}
\end{table}

\section{Hyper-parameters for GLUE Experiments}\label{glueparam}

\begin{table}[htbp]
\caption{Hyper-parameters for NLP experiments}
\label{nlp-params}
\begin{footnotesize}
\begin{sc}
\begin{center}
\begin{tabular}{ccccc}
\toprule
Task & COLA & MRPC & SST-2 & QQP
\\ \midrule
Optimizer     & \multicolumn{4}{c}{AdamW}\\
Learning Rate       & $10^{-5}$ & $10^{-5}$  &$5\times 10^{-6}$ & $2\times 10^{-5}$\\
Learning Rate Schedule       & \multicolumn{4}{c}{One cycle with $6\%$ warm-up}\\
Number of Epochs        & 60 & 60  & 20 & 15\\
Weight Decay & \multicolumn{4}{c}{0.01}\\
$\rho$ (SAM/mSAM)      &0.01 & 0.01  & 0.05 & 0.05\\
\bottomrule
\end{tabular}
\end{center}
\end{sc}
\end{footnotesize}
\end{table}

Our experiments in this section are done on four Nvidia V100 GPUs, with an effective batch size of $32$. For mSAM, we have used the micro-batch size of four, which corresponds to $m=8$. The other hyper-parameters are listed in Table~\ref{nlp-params}.

\end{document}